\documentclass[journal]{IEEEtranTIE}
\usepackage{cite}
\usepackage{amsmath,amssymb,amsfonts}
\usepackage{graphicx}
\usepackage{textcomp}
\usepackage{xcolor}
\def\BibTeX{{\rm B\kern-.05em{\sc i\kern-.025em b}\kern-.08em
    T\kern-.1667em\lower.7ex\hbox{E}\kern-.125emX}}

\usepackage{kotex}
\usepackage{picinpar}
\usepackage{url}
\usepackage{colortbl}
\usepackage{soul}
\usepackage{multirow}
\usepackage{pifont}
\usepackage{alltt}
\usepackage{enumerate}
\usepackage{pbox}
\usepackage{epsfig} % for postscript graphics files
\usepackage{times} % assumes new font selection scheme installed
\usepackage{amsmath} % assumes amsmath package installed
\usepackage{amssymb}  % assumes amsmath package installed
\graphicspath{{./fig/}}
\usepackage{siunitx}
\usepackage{xcolor}
\usepackage{enumitem}
\usepackage{algorithm}
\usepackage{algpseudocode}
\usepackage[caption=false,font=footnotesize]{subfig} 
\usepackage{bm}
\usepackage{amsthm}
\theoremstyle{plain}
\newtheorem{theorem}{Theorem}

\newcommand{\rom}[1]{\uppercase\expandafter{\romannumeral #1\relax}}

\title{\LARGE \bf Frequency Response Data-Driven Disturbance Observer Design for Flexible Joint Robots}
\author{Deokjin Lee, Junho Song, Alireza Karimi, and Sehoon Oh% <-this % stops a space
\thanks{$^{1}$Deokjin Lee, Junho Song and Sehoon Oh are with Department of Robotics Engineering,
        DGIST, Daegu, Korea 42988.
        (e-mail : {\tt\footnotesize \{djlee, optimus120, sehoon\}@dgist.ac.kr})
        (phone : +82-53-785-6209)
        }

\thanks{$^{2}$Alireza Karimi ({\tt alireza.karimi@epfl.ch}) is with the Automatic Control Laboratory, École Polytechnique Fédérale de Lausanne (EPFL), Switzerland.}
}

\begin{document}

\maketitle
\thispagestyle{empty}
\pagestyle{empty}

%%%%%%%%%%%%%%%%%%%%%%%%%%%%%%%%%%%%%%%%%%%%%%%%%%%%%%%%%%%%%%%%%%%
\begin{abstract}
Motion control of flexible joint robots (FJR) is challenged by inherent flexibility and configuration-dependent variations in system dynamics. While disturbance observers (DOB) can enhance system robustness, their performance is often limited by the elasticity of the joints and the variations in system parameters, which leads to a conservative design of the DOB. This paper presents a novel frequency response function (FRF)-based optimization method aimed at improving DOB performance, even in the presence of flexibility and system variability. The proposed method maximizes control bandwidth and effectively suppresses vibrations, thus enhancing overall system performance. Closed-loop stability is rigorously proven using the Nyquist stability criterion. Experimental validation on a FJR demonstrates that the proposed approach significantly improves robustness and motion performance, even under conditions of joint flexibility and system variation.
\end{abstract}
%%%%%%%%%%%%%%%%%%%%%%%%%%%%%%%%%%%%%%%%%%%%%%%%%%%%%%%%%%%%%%%%%%%


\markboth{IEEE TRANSACTIONS ON INDUSTRIAL ELECTRONICS}%
{}

\section{INTRODUCTION} \label{sec:1}
\IEEEPARstart{T}{he} demand for flexible joint robots (FJR) has rapidly grown, driven by their critical roles in human-robot interaction and operational safety~\cite{de2016robots}. However, controlling FJR is challenging due to their inherently complex dynamics, configuration-dependent variations, and vibrations caused by joint elasticity~\cite{spong1987modeling}. These factors restrict achievable control bandwidth and compromise robustness.

Various robust control methodologies, such as $\mathcal{H}_{\infty}$ control \cite{makarov2016modeling}, sliding mode control~\cite{rsetam2019cascaded}, and disturbance observer (DOB)-based methods~\cite{chen2015disturbance}, have been extensively studied to address these challenges. Among these, DOB stands out due to its simplicity and practical effectiveness. DOB approaches have successfully addressed issues like vibration suppression~\cite{bang2009robust}, friction compensation~\cite{kim2019model}, and decoupling~\cite{nozaki2013decoupling}. Nevertheless, extending traditional DOB strategies to multi-input-multiple-output (MIMO) robotic systems introduces additional considerations, including system variation, coupling effect, and dynamic complexity.

Recent advances in optimization-based DOB designs have systematically addressed these complexities. In particular,  time-domain optimization methods~\cite{yun2024disturbance, xie2021disturbance} have enhanced disturbance rejection capabilities by incorporating explicit constraints into the optimization process. However, these approaches have predominantly focused on error minimization. They do not adequately capture critical robotic system characteristics, such as resonance and configuration-dependent dynamic variations, which are essential for effective vibration suppression in FJR.

\begin{figure}
\begin{center}
    \includegraphics[width=1.0\columnwidth]{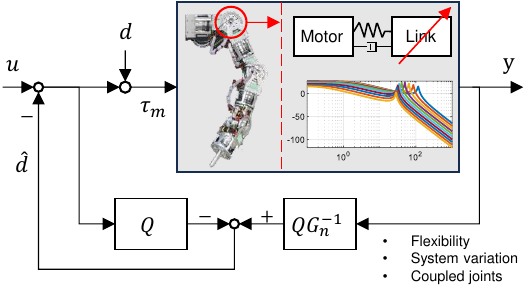}
	\caption{Disturbance observer framework for a flexible joint robot. The robot dynamics vary with configuration, leading to resonance changes that the observer must account for.}
	\label{fig:Fig1_Robot_DOB}
\end{center}
\vspace{-0.2in}
\end{figure}

To address these challenges, frequency-domain convex optimization methods have emerged, utilizing closed-loop transfer functions to comprehensively capture resonance behaviors and accurately represent system variations~\cite{boyd2016mimo}. Such frequency-domain convex optimization strategies have also been extended to DOB design for MIMO robotic systems, optimizing both the nominal system models and associated Q-filters~\cite{tesfaye2000sensitivity}. However, the effectiveness of these model-based optimization approaches is limited by their reliance on precise system identification, which becomes increasingly difficult in the presence of nonlinearities, such as joint friction and unmodeled dynamics.

In response to these limitations, Frequency Response Function (FRF)-based optimization methods have emerged, enabling automated tuning of DOB parameters according to specified performance criteria, without relying on model estimation~\cite{karimi2010fixed}. Such methods have shown promise in flexible systems, particularly by optimizing control bandwidth while systematically addressing resonance-induced constraints~\cite{wang2022frequency}. However, existing FRF-based optimization approaches predominantly address single-input-single-output (SISO) systems and neglect significant system variations. In robotic applications, configuration-dependent system variations critically influence system performance, necessitating explicit consideration within controller design frameworks. Although research in~\cite{schuchert2023frequency} has accounted for system variations, it has primarily emphasized task-space performance without adequately addressing joint-level dynamics and vibration suppression, which are fundamental control aspects for robotic joints.

Motivated by the limitations in existing methods, this paper proposes a novel FRF-based, data-driven optimization framework employing Linear Matrix Inequalities (LMI) specifically tailored for DOB design in FJR. The proposed method significantly improves control bandwidth, ensures robust vibration suppression, incorporates necessary stability margins, and adaptively addresses inertia variations across various robotic configurations.

The main contributions of this paper are summarized as follows:
\begin{itemize}
\item A data-driven DOB optimization methodology for FJR is proposed, utilizing FRF data to effectively handle unpredictable nonlinearities and system variations at the joint level, which is a fundamental control layer.
\item  The DOB optimization is performed within a unified framework to maximize control bandwidth and effectively suppress vibrations, resulting in faster and more robust motion in robotic systems.
\item Comprehensive experimental validation of the proposed DOB is conducted on both single-joint and multi-joint robot systems under various disturbance conditions, demonstrating practical effectiveness and robustness.
\end{itemize}

The remainder of this paper is organized as follows: Section \ref{sec:2} discusses critical challenges in DOB design specific to FJR. Section \ref{sec:3} introduces the proposed optimization methodology. Section \ref{sec:4} details the convexification approach for the proposed non-convex optimization problem. Section \ref{sec:5} presents experimental validation results, and finally, Section \ref{sec:6} summarizes the conclusions and outlines future research directions.

\section{Problem Formulation} \label{sec:2}
\subsection{Flexible Joint Robot}
The dynamics of FJR can be derived based on Spong's assumption~\cite{spong2020robot}:
\begin{eqnarray}
    \mathcal{M}(q)\Ddot{q}+\mathcal{C}(q,\Dot{q})\Dot{q} + \mathcal{G}(q) &=&\tau_s, \\
    \mathcal{B}\Ddot{\theta}+\mathcal{D}\Dot{\theta}-\tau_s &=& \tau_m,
    \label{eq:FJR dynamics}
\end{eqnarray}
where $q, \theta \in \mathbb{R}^m$ are the link and motor positions, respectively, and $m$ is the number of joints. The elastic joint torque is modeled as $\tau_s = \mathcal{K}(q - \theta)$, with $\mathcal{K} \in \mathbb{R}^{m \times m}$ denoting the joint stiffness matrix. The matrices $\mathcal{M}$, $\mathcal{C}$, $\mathcal{B}$, $\mathcal{D} \in \mathbb{R}^{m \times m}$ represent the link inertia, Coriolis/centrifugal effects, motor inertia, and motor damping, respectively. The vector $\mathcal{G} \in \mathbb{R}^m$ accounts for gravitational torques, and $\tau_m \in \mathbb{R}^m$ is the control input.

Due to the inherent complexity of configuration-dependent dynamics, practical controller design typically employs linear time-invariant (LTI) approximations at specific operating points. At these operating points, the robot dynamics simplify to a discrete-time LTI representation:
\begin{equation}
Y(z) = C_l(zI - A_l)^{-1}B_l U(z) + D_l U(z),
\end{equation}
where $G(z) = C_l(zI - A_l)^{-1}B_l + D_l$ represents the linearized LTI transfer function at a particular operating point. The corresponding frequency response at each operating condition is given by:
\begin{equation}
Y(e^{j\omega}) = G(e^{j\omega})U(e^{j\omega}),
\end{equation}
where $j = \sqrt{-1}$ is the imaginary unit and $\omega \in \Omega := (0, \pi/T_s]$, with $T_s$ denoting the sampling period.

Since the joint dynamics are structurally identical across all degrees of freedom, a single representative joint is considered for the remainder of the analysis without loss of generality.

\subsection{Data-driven Disturbance Observer Design}

\begin{figure}
	\begin{center}
		\subfloat[]{\includegraphics[width=0.45\textwidth]{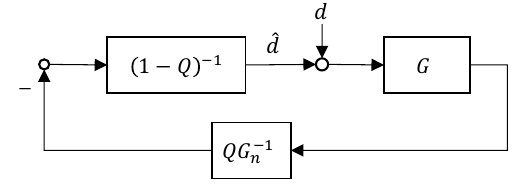}}\\	
        \vspace{-1em}
        \subfloat[]{\includegraphics[width=0.45\textwidth]{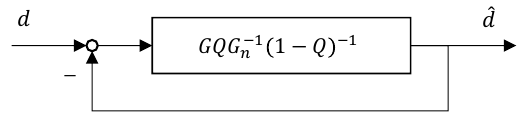}}
        \caption{Reformulated disturbance observer structures from Fig.~\ref{fig:Fig1_Robot_DOB}. (a) Analytical form used in deriving the optimization framework. (b) Closed-loop configuration for FRF data-driven optimization.}
        \label{fig:Reformulated_DOB}
	\end{center}
\end{figure}

Conventional DOB, depicted in Fig.~\ref{fig:Fig1_Robot_DOB}, estimates disturbances based on a nominal plant model $G_n$ and a low-pass filter $Q$. However, accurately identifying the nominal model is challenged by unmodeled dynamics and system variations. 

To overcome these limitations, a data-driven formulation is adopted in which both the nominal model $G_n$ and the $Q$ filter are optimized simultaneously. For this purpose, the traditional DOB structure is reformulated—illustrated in Fig.~\ref{fig:Reformulated_DOB}—as a closed-loop system described in terms of the actual disturbance $d$, the estimated disturbance $\hat{d}$. This structure introduces a loop gain $L$, given by:
\begin{align}
    L(z,h,t) 
    &= G(z)Q(z)G_{n}(z)^{-1}(1-Q(z))^{-1},\nonumber\\
    &= G(z)\frac{h_{qn}z^{qn}+\dots+h_0}{t_{qd}z^{qd}+\dots+t_0}\\
    &= G(z)\frac{N(z,h)}{D(z,t)}
    \label{eq:def_L}  
\end{align}
where $h = [h_{0},\dots, h_{pn}]^T$ and $t = [t_{0}, \dots, t_{pd}]^T$ are optimization parameters. This formulation preserves the dynamic characteristics of conventional DOB while enabling data-driven co-optimization of the nominal model and $Q$ filter.

Using the loop gain $L$, the corresponding closed-loop transfer functions are defined as:
\begin{align}
    T(e^{j\omega},h,t) 
    &=\frac{L(e^{j\omega},t,h)}{1+L(e^{j\omega},h,t)},\nonumber\\
    &=\frac{G(e^{j\omega})N(e^{j\omega}, h)}{D(e^{j\omega},t)+G(e^{j\omega})N(e^{j\omega},h)},     \label{eq:def_T}\\
    S(e^{j\omega},h,t)&＝1-T(e^{j\omega},t,h).
    \label{eq:def_S}
\end{align}
where $T$ quantifies the accuracy of disturbance estimation, while $S$ evaluates the robustness and performance of the system. The DOB design is then formulated as a convex optimization problem over the FRF data, with $h$ and $t$ as optimization variables.

\subsection{Challenges in DOB Design for Flexible Joint Robots} \label{sec:2.3}
\subsubsection{Flexibility and System Variation}
Designing an effective DOB for FJR demands addressing critical challenges, including joint flexibility, system variations, and coupling effects. As shown in Fig.~\ref{fig:Fig1_Robot_DOB}, joint flexibility induces resonance phenomena, generating vibrations that severely limit control bandwidth. Variations in link inertia across different configurations further shift resonance frequencies, adversely affecting control performance. Conventional model-based methods, reliant on FRF estimation, struggle due to unmodeled nonlinearities, emphasizing the necessity of data-driven approaches capable of dynamically adapting to these complexities.

\subsubsection{Conservative DOB design}
Traditional DOB designs are typically conservative to accommodate unmodeled nonlinearities and variable resonance frequencies. Such cautious parameter selections significantly restrict control bandwidth, resulting in suboptimal performance. While existing data-driven optimization methods address resonance and system variation using FRF data, controller performance is often limited by predefined constraints. Thus, optimizing control performance—specifically maximizing control bandwidth and minimizing overshoot—is essential.

Therefore, addressing these challenges necessitates an FRF data-driven DOB optimization framework capable of dynamically enhancing control performance.

\section{Optimization of Disturbance Observer for Flexible Joint Robot} \label{sec:3}
This section focuses on optimizing the DOB for the FJR by addressing key challenges, including vibration suppression, model variations, and the conservative design considerations discussed in Section~\ref{sec:2.3}. The goal is to develop an automatically tuned DOB that can adapt to dynamic changes in the robot configuration, maximize performance, suppress vibrations, and ensure robust stability.

\subsection{Constraints for Optimization}
The DOB optimization performance is guided by three principal constraints: stability margins, sensitivity function, and complementary sensitivity function. These constraints ensure robust stability, effective disturbance attenuation, and high-quality tracking performance, thereby forming the foundational criteria for optimal DOB design.

\subsubsection{Stability Margins}
Ensuring sufficient stability margins is critical to the robustness and reliability of the DOB. Stability margins are enforced through sensitivity function constraints represented as:
\begin{eqnarray}
    &W_{1} = \sigma, \\
    &|W_{1}S(e^{j\omega},h,t)| \leq 1 \qquad \forall \omega\in\Omega,
    \label{eq:Complementary and Sensitivity}
\end{eqnarray}
where the weighting function $W_1(\omega)$ is derived from the modulus margin, defined as the inverse of the infinity norm of the sensitivity function~\cite{garcia2004robust}. The parameter $\sigma$ is carefully selected to meet the required robustness criteria, ensuring stable and predictable system behavior under varying operational conditions and uncertainties.

\subsubsection{Sensitivity function}
The sensitivity function directly influences the DOB's ability to attenuate disturbances and achieve high-performance motion control. The shaping of the sensitivity function is accomplished by employing the weighting function:
\begin{eqnarray}
    & W_{2}(\omega,\zeta) = \frac{\zeta}{j\omega}, \\
    &|W_{2}(\omega,\zeta)S(e^{j\omega},h,t)| \leq 1 \qquad \forall \omega\in\Omega,
    \label{eq:Complementary and Sensitivity}
\end{eqnarray}
where $W_2$ functions as a first-order differentiator incorporating the optimization parameter $\zeta$. This approach maximizes the bandwidth of the sensitivity function, as depicted in Fig.~\ref{fig:S}, thereby improving disturbance rejection capabilities. By tuning the parameter $\zeta$, the DOB is optimized to balance maximum achievable bandwidth with the required stability constraints.

\begin{figure}
\begin{center}
	\includegraphics[width=1.0\columnwidth]{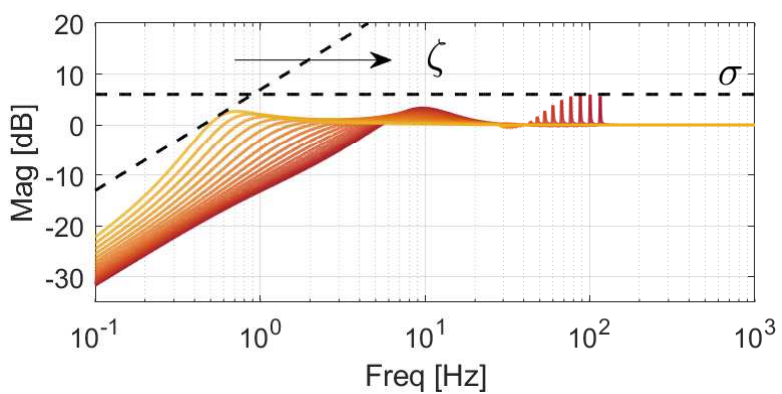}
	\caption{Sensitivity function of DOB under system variation, illustrating bandwidth $\zeta$ maximization while maintaining stability margin $\sigma$.}
	\label{fig:S}
\end{center}
\vspace{-0.2cm}
\end{figure}

\begin{figure}
\begin{center}
	\includegraphics[width=1.0\columnwidth]{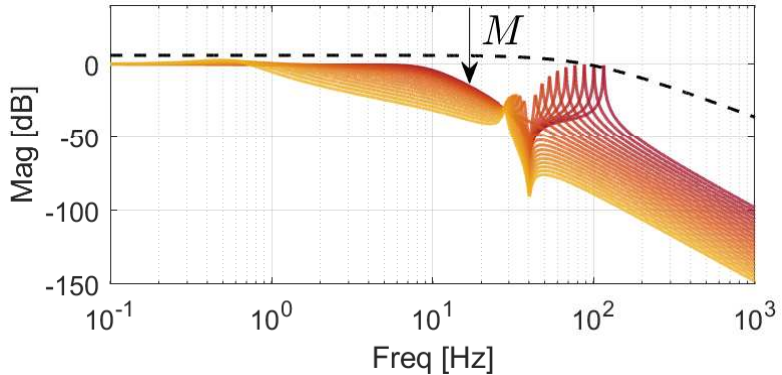}
	\caption{Complementary sensitivity function of DOB under system variation, demonstrating vibration suppression through overshoot $M$ minimization.}
	\label{fig:T}
\end{center}
\vspace{-0.2cm}
\end{figure}

\subsubsection{Complementary Sensitivity function}
Constraints on the complementary sensitivity function ensure the DOB achieves superior control performance with minimal tracking error and reduced overshoot:
\begin{eqnarray}
    &W_{3}(\omega, M) = M(\tau(j\omega)+1)^n,\\
    &|W_{3}(\omega,M)T(e^{j\omega},h,t)| \leq 1  \qquad \forall \omega\in\Omega,
    \label{eq:Complementary and Sensitivity}
\end{eqnarray}
where the weighting function $W_3$ is structured as an $n$-th order low-pass filter, parameterized by the preselected filter order $n$ and time constant $\tau$. Minimizing the optimization parameter $M$ is critical in reducing system overshoot and vibrations as demonstrated in Fig.~\ref{fig:T}, thereby enhancing the overall smoothness and precision of robot movements.

\subsection{Optimization Problem Formulation} \label{sec:3.4}
The optimization seeks to simultaneously maximize control bandwidth and minimize vibration, inherently managing the trade-off between these objectives. Bandwidth enhancement is achieved by maximizing $\zeta$, while minimizing vibrations and overshoot is addressed through maximizing parameter $M$. To balance these objectives, a positive weighting parameter $\alpha > 0$ is introduced, allowing designers to tune the trade-off according to specific application requirements. Without loss of generality, only one weighting parameter is used, as the other can be normalized. The resulting optimization problem is formulated as:
\begin{subequations} 
\begin{align}
    & \underset{h,\; t}{\text{maximize}} 
    & & \zeta + \alpha M \\
    & \text{subject to}
    & & 0 < \zeta < \tau^{-n},  \; M\leq 1, \label{eq:nonconv_constraint1} \\
    &&& \left\vert W_{1} S(e^{j\omega}, h, t) \right\vert \leq 1 & \forall \omega\in\Omega, 
    \label{eq:nonconv_S1} \\
    &&& \left\vert W_{2}(\omega, \zeta) S(e^{j\omega}, h, t) \right\vert \leq 1 & \forall \omega\in\Omega,
    \label{eq:nonconv_S2} \\
    &&& \left\vert W_{3}(\omega, M) T(e^{j\omega}, h, t) \right\vert \leq 1 & \forall \omega\in\Omega. \label{eq:nonconv_T}
\end{align}
\end{subequations}

This optimization formulation explicitly accounts for model uncertainties and ensures the DOB optimal performance across various robot configurations. The subsequent section addresses the convexification of this non-convex optimization problem to facilitate efficient and reliable computational solutions.

\section{Convexification of Constraints} \label{sec:4}
The non-convex constraints detailed in Section \ref{sec:3} are reformulated into a convex optimization framework employing second-order cone constraints and LMI. This approach facilitates efficient and robust computational solutions using semidefinite programming (SDP) methods.

\subsection{Second-Order Cone Constraints: LMI-Based Reformulation} \label{sec:4.1}
The non-convex constraints from Section~\ref{sec:3.4} can initially be represented in the general form:
\begin{equation} 
|F|^2 \leq \gamma P^*P, \quad \gamma > 0, 
\label{eq:convex_con1} 
\end{equation}
where $F \in \mathbb{C}^n$ and $P \in \mathbb{C}$ represent linear functions of the optimization parameters. This form is inherently non-convex. To achieve a convex reformulation, an inner convex approximation is introduced through an arbitrary selection $P_c \in \mathbb{C}$:
\begin{equation} 
    P^{*}P \geq P^{*}P_{c} + P_{c}^{*}P - P_c^{*}P_{c},
    \label{eq:convex_con2} 
\end{equation} 
which follows from $(P-P_c)^*(P-P_c)\geq0$. Applying this approximation converts~\eqref{eq:convex_con1} into a convex approximation:
\begin{equation} 
    |F|^2 < \gamma (P^{*}P_{c} + P_{c}^{*}P - P_c^{*}P_{c}).
    \label{eq:convex_con2} 
\end{equation} 

To simplify the notation, define: 
\begin{equation} 
    \Phi=P^{*}P_{c} + P_{c}^{*}P - P_c^{*}P_{c}.
    \label{eq:convex_con3} 
\end{equation}

This leads to the definition of a second-order rotated cone: 
\begin{equation} 
    \mathbb{C}_r = \{(F, \gamma, \Phi) \mid |F|^2 \leq \gamma\Phi, \gamma, \Phi > 0\},
    \label{eq:convex_con4} 
\end{equation}
which can be embedded within the cone of positive semidefinite matrices, enabling an equivalent LMI representation:
\begin{equation}
\left\{
    \begin{aligned}
         & |F|^2 \leq \gamma P^*P, \\
         & \gamma, \Phi > 0
    \end{aligned}
\right.
\quad
\Leftarrow
\quad
\begin{bmatrix} 
    \Phi & F \\
    F^T & \gamma
\end{bmatrix}
    \succeq 0.
\label{eq:SSE}
\end{equation}

The LMI reformulation enables the use of SDP solvers to efficiently handle the convexified constraints.

\subsection{Constraints for Stability Margin} \label{sec:4.2}
The stability margin constraint introduced in~\eqref{eq:nonconv_S1}:
\begin{eqnarray}
    |W_{1}S| \leq 1.
    \label{eq:ws_constraint}
\end{eqnarray}

Omitting the dependency on $\omega$ for brevity, the original non-convex inequality can be described in second-order cone constraints as~\eqref{eq:convex_con1}:
\begin{eqnarray}
    |\sigma D|^2 \leq P^*P,
\end{eqnarray}
in which 
\begin{eqnarray}
    P=D+GN\qquad and\qquad   P_c=D_c+GN_c.
\end{eqnarray}

Using the inner approximation \eqref{eq:convex_con2} and defining $\Phi$ via \eqref{eq:convex_con3}, the following convex inequality is obtained. 
\begin{eqnarray}
    |\sigma D|^2 \leq \Phi. 
\end{eqnarray}

The original nonlinear constraint \eqref{eq:nonconv_S1} can thus be transformed into the following sufficient LMI condition:
\begin{eqnarray}
\begin{bmatrix} 
         \Phi & \sigma D \\
        \sigma D^* & 1
    \end{bmatrix}
    \succeq 0.
    \label{eq:ws_constraint_final}
\end{eqnarray}

\subsection{Constraints for Sensitivity Function} \label{sec:4.3}
The sensitivity function constraint defined in~\eqref{eq:nonconv_S2}:
\begin{eqnarray}
    |W_{2}S| \leq 1. 
    \label{eq:ws_constraint}
\end{eqnarray}

The original non-convex inequality is reformulated to satisfy  convex-cone constraints, as discussed in Section~\ref{sec:4.2}:
\begin{eqnarray}
    \left|\frac{\zeta}{j\omega}\right|^2 |D|^2 \leq \Phi.
    \label{eq:LMI_S_2}
\end{eqnarray}

However, the presence of $|\zeta D|^2$ renders the formulation non-convex. To resolve this, an auxiliary optimization variable $\gamma_1$ is introduced, which satisfies the following inequality derived using the Schur complement~\cite{boyd1994linear}:
\begin{eqnarray}
    0 < \gamma_1 \leq 2\zeta_c^{-2} - \zeta_{c}^{-4} \zeta^2 \leq \zeta^{-2},
\end{eqnarray}
which follows from $(\zeta^{-2}-\zeta_c^{-2})(\zeta^{-2}-\zeta_c^{-2})\geq0$ for some arbitrary $\zeta_c$. The auxiliary constraint is enforced by the following inequality:
\begin{eqnarray}
    \begin{bmatrix} 
        2\zeta_{c}^2 - \gamma_{1} \zeta_c^4 & \zeta \\
        \zeta & 1
    \end{bmatrix}
    \succeq 0.
\end{eqnarray}

Incorporating the auxiliary constraint results in a convex constraint expressed by the following inequalities:
\begin{eqnarray}
    \left|\frac{\zeta}{\omega}\right|^2 |D|^2 \leq \frac{|D|^2}{|\omega|^2 \gamma_1} \leq \Phi.
\end{eqnarray}

Thus, the sufficient condition for the constraint~\eqref{eq:nonconv_S2} is formulated as LMI:
\begin{eqnarray}
\begin{bmatrix} 
         \Phi & D \\
        D^* & \omega^2 \gamma_1
    \end{bmatrix}
    \succeq 0.
    \label{eq:ws_constraint_final}
\end{eqnarray}

\subsection{Constraints for Complementary Sensitivity Function} \label{sec:4.4}
The complementary sensitivity constraint from~\eqref{eq:nonconv_T}:
\begin{eqnarray}
    |W_3 T| \leq 1,
    \label{eq:convex_reform}
\end{eqnarray}

Introducing $Y=|(\tau(j\omega) + 1)^n N|^2$, the convex optimization is obtained with the inner approximation from \eqref{eq:convex_con2}:
\begin{eqnarray}
    |MY|^2 \leq \Phi
\end{eqnarray}
To ensure convexity, the additional inequality on $\gamma_2$ is derived using the Schur complement:
\begin{eqnarray}
    \begin{bmatrix} 
        2M_c^2 - \gamma_{2} M_c^4 & M \\
        M & 1
    \end{bmatrix}
    \succeq 0,
\end{eqnarray}
with the condition $0 < \gamma_2 \leq 2M_c^{-2} - M_c^{-4} M^2 \leq M^{-2}$ for some arbitrary variable $M_c$. Thus, the reformulation becomes:
\begin{eqnarray}
   |MY|^2 \leq \frac{|Y|^2}{\gamma_2} \leq \Phi.
\end{eqnarray}

Finally, the original non-convex constraint in~\eqref{eq:nonconv_T} is then reformulated as:
\begin{eqnarray}
    \begin{bmatrix} 
        \Phi & Y\\
        Y^* & \gamma_2
    \end{bmatrix}
    \succeq 0,
    \label{eq:convex_reform_final}
\end{eqnarray}
which ensures the convex formulation of the problem.

\subsection{Closed-Loop Stability} \label{sec:4.5}
Closed-loop stability can be guaranteed based on the generalized Nyquist stability criterion. 
\begin{theorem}
    The closed-loop system with the plant model $G(z)$ and the controller $K(z)=ND^{-1}$ is stable if and only if the Nyquist plot of $1+G(z)K(z)$ satisfies the following conditions:
    \begin{enumerate}
        \item The Nyquist plot makes $N_G$ and $N_K$ counterclockwise encirclements of the origin, where $N_G$ and $N_K$ represent the number of poles of $G(z)$ and $K(z)$ outside the unit circle, respectively.
        \item The plot does not pass through the origin.
    \end{enumerate}
\end{theorem}

\begin{theorem}
    Given a plant $G$, an initial stabilizing controller $K_c=N_cD_c^{-1}$ with $D_c\neq 0, \forall\omega,$ and feasible solutions $N$ and $D$ satisfying the following inequality:
    \begin{equation}
        P^{*}P_c + P_c^*P > 0 \qquad \forall \omega\in\Omega,
    \end{equation}
    then the controller $K(z)$ stabilizes the closed-loop system if: 
    \begin{enumerate}
        \item $D\neq 0, \forall\omega$.
        \item The initial controller $K_c$ and the final controller $K$ share the same poles on the stability boundary.
        \item The order of $D$ is equal to the order of $D_c$
    \end{enumerate}
\end{theorem}

\begin{proof}
    See~\cite{karimi2017data}.
\end{proof}

Theorem~1 holds if $P(e^{j\omega})$ and $P_c(e^{j\omega})$ have the same winding number (wno) around the origin when $\omega$ traverses the Nyquist contour:
    \begin{eqnarray}
        wno(P)=wno(P_c),
        \label{eq:wno_P1}
    \end{eqnarray}
which is guaranteed by convexifing $P^*P$ around $P_c$, as $\Phi>0$ enforces: 
\begin{equation}
    P^*P_c+P_c^*P>0.
    \label{eq:wno_P2}
\end{equation}

For Theroem~2, $K$ and $K_c$ have the same number of stable poles if:
\begin{equation}
    wno(D) = wno(D_c).
    \label{eq:wno_D2}
\end{equation}
which is guaranteed by the constraint:
\begin{equation}
    D^*D_c + DD^*_c>0 \;\;\; \forall \omega\in \Omega.
    \label{eq:wno_D3}
\end{equation}
This ensures Condition~1 in Theorem~2. If the initial controller $K_c$ is stable, it guarantees that $K$ is also stable, thereby satisfying Condition~2 in Therorem~2. Condition~3 in Theorem~2 is generally not restrictive, as any lower-order initial controller can be augmented by adding zeros and poles at the origin in $N$ and $D$, ensuring the condition without affecting the initial controller.

\subsection{Convex Optimization Formulation} \label{sec:4.6}
Combining all convexified constraints, the optimization problem detailed in Section~\ref{sec:3.4} is presented in the standard LMI form:
\begin{subequations}
\label{eq:FinalLMIform}
\begin{align}
\underset{h,\;t}{\text{maximize}} 
\quad & \,\zeta \;+\;\alpha\,M 
\label{eq:cost_convex}\\
\text{subject to} \quad
& 0<\zeta<\tau^{-n}, \\
& \gamma_1,\gamma_2>0, \quad 0<M\le 1,
\label{eq:constraint_basic}\\
& \Phi \;>\;0,\quad D^*\,D_c \;+\;D\,D_c^* > 0,
\label{eq:constraint_stability}\\
&\begin{bmatrix} 
         \Phi & \sigma D \\
        \sigma D^* & 1
    \end{bmatrix}
    \succeq 0,
\label{eq:constraint_margin}\\
& \begin{bmatrix}
\Phi & D\\[3pt]
D^* & \omega^2\,\gamma_1
\end{bmatrix}
\succeq 0,\quad
\begin{bmatrix}
2\,\zeta_c^2 - \gamma_1\,\zeta_c^4 & \zeta\\[3pt]
\zeta & 1
\end{bmatrix}
\succeq 0,
\label{eq:constraint_sens}\\
& \begin{bmatrix}
\Phi & Y\\[3pt]
Y^* & \gamma_2
\end{bmatrix}
\succeq 0,\quad
\begin{bmatrix}
2\,M_c^2 - \gamma_2\,M_c^4 & M\\[3pt]
M & 1
\end{bmatrix}
\succeq 0,
\label{eq:constraint_compsens}
\end{align}
\end{subequations}
where each LMI condition applies across all relevant frequencies \(\omega\in\Omega\). The presented LMI formulation systematically incorporates stability criteria as well as sensitivity and complementary sensitivity constraints, directly reflecting the optimization objectives of maximizing bandwidth $\zeta$ and minimizing vibrations $M$. The optimization can be efficiently solved using standard SDP solvers, such as MOSEK, through iterative refinement of linearization points \(P_c\), \(\zeta_c\), and \(M_c\) until convergence.

\section{Validation} \label{sec:5}
The proposed method is validated using a 7-DOF FJR~\cite{lee2023exsler}. FRF data representing the linearized dynamics for each link are acquired and utilized for DOB optimization. The performance of the optimized DOB is evaluated against system variations and impact tolerance, and compared with a conventional model-based DOB~\cite{sariyildiz2018stability}.

\subsection{Experimental Setup} \label{sec:5.1}
The experimental platform is a 7-DOF FJR, depicted in Fig.~\ref{fig:Dual_arm_hug}, specifically focusing on the left arm. Each joint of the robot is equipped with both motor-side and load-side encoders, along with torque sensors. Real-time robot control is implemented through EtherCAT communication, operating reliably at a sampling frequency of 1 kHz. Joint flexibility in this platform predominantly originates from the gear transmission systems and integrated torque sensors, which significantly influence the robot's dynamic behavior.

The robot's motion ranges are constrained by physical hardware limitations and operational safety criteria, as detailed in Table~\ref{tbl:robot_motion_range}. Additionally, the inertia associated with each link varies substantially throughout the defined motion ranges. This variation is approximated using CAD software~\cite{SolidWorks2022}, highlighting significant inertia variations of approximately 100-fold, 400-fold, and 50-fold for the first, second, and third joints, respectively.

\begin{figure}[tb]
\begin{center}
	\includegraphics[width=0.7\columnwidth]{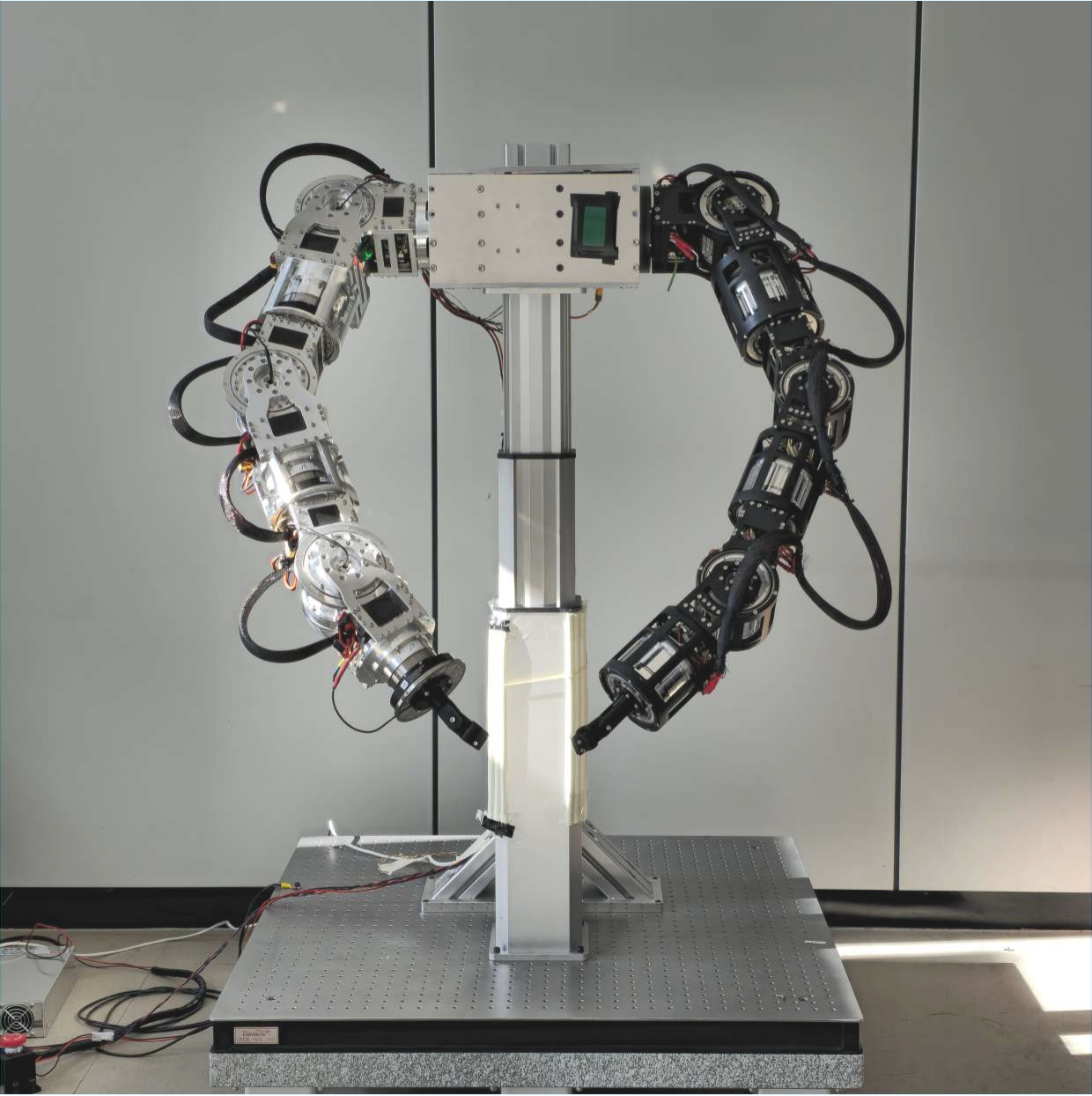}
	\caption{Configuration of the flexible joint robot. The left arm is utilized in this experiment.}
	\label{fig:Dual_arm_hug}
\end{center}
\vspace{-0.5cm}
\end{figure}

\begin{table}[t!]
	\centering
    \caption{Joint motion ranges and corresponding inertia variations for the 7-DOF flexible joint robot (left arm in Fig.~\ref{fig:Dual_arm_hug})}.
    \label{tbl:robot_motion_range}
	\begin{tabular}{c|ccc}
		\hline\hline
        Joint & Motion Range (deg) & Inertia Variation (kg$\cdot$m$^2$) \\  
        \hline
        1 & $-90$ $\sim$ $90$    & 0.03 $\sim$ 3.98 \\
        2 & $-86$ $\sim$ $16$    & 0.01 $\sim$ 4.15 \\
        3 & $-138$ $\sim$ $155$  & 0.014 $\sim$ 0.69 \\
        4 & $-80$ $\sim$ $108$   & 0.50 $\sim$ 0.70 \\
        5 & $-138$ $\sim$ $138$  & 0.008 $\sim$ 0.09 \\
        6 & $-97$ $\sim$ $108$   & 0.07 \\
        7 & $-172$ $\sim$ $172$  & 0.0002 \\
    \hline\hline
	\end{tabular}
\end{table}

\subsection{System Identification} \label{sec:5.2}
FRF data are acquired through system identification~\cite{pintelon2012system}, performed across multiple robot configurations. Specifically, a Schroeder multisine torque input is applied individually to each joint, and the resulting load-side velocity responses are measured. During these experiments, all other joints are held stationary using closed-loop position control to isolate the dynamics of the targeted joint. 

Identification results, illustrated in Fig.~\ref{fig:J2_FRF}, highlight a prominent resonance mode within the frequency range of 10–20 Hz attributable to joint flexibility, with variations dependent on robot configurations. Additionally, anti-resonance modes are identified in the range of 8–10 Hz. Notably, the DC gain also varies across configurations, increasing in correlation with friction.

Although conventional theoretical analysis anticipates only variations in resonance frequency due to link inertia changes, the FRF measurements exhibit significant deviations in both DC gain and anti-resonance characteristics. These empirical insights highlight the effectiveness of FRF-based DOB design in enhancing the accuracy and robustness of robot control performance.

\begin{figure}
	\begin{center}
		\subfloat{\includegraphics[width=0.5\textwidth]{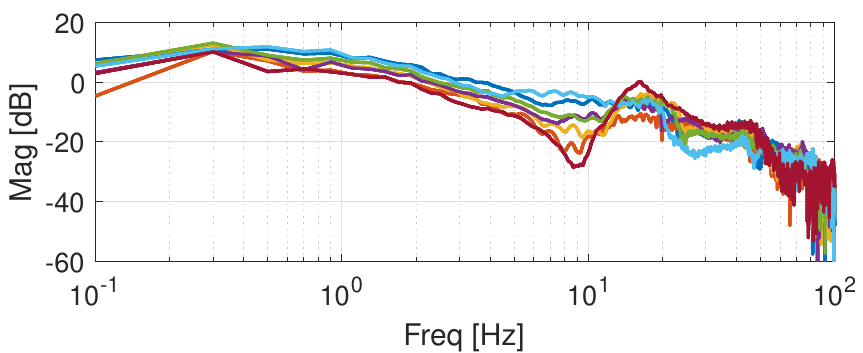}}\\	
        \vspace{-2em}
        \subfloat{\includegraphics[width=0.5\textwidth]{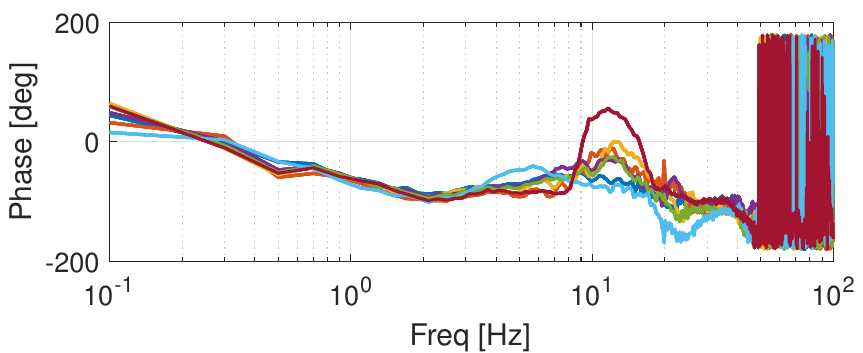}}
		\caption{Identified FRF for the 2nd joint of the flexible joint robot, illustrating variations at different operating points. Changes in anti-resonance and resonance modes are observed.}
        \label{fig:J2_FRF}
	\end{center}
\end{figure}

\subsection{DOB optimization for 7-DOF FJR} \label{sec:5.3}
DOB optimization is performed based on identified FRF data as shown in Fig.~\ref{fig:J2_FRF} to maximize control bandwidth and minimize overshoot for each joint, with a detailed focus on the second joint due to its representative dynamics. The optimized DOB performance is compared against a model-based DOB that uses a nominal two-mass system model, estimated at the median inertia.

Optimization results, shown in Fig.~\ref{fig:Fig_T_opt}, demonstrate that the optimized DOB achieves greater control bandwidth and reduced overshoot compared to the model-based DOB, which exhibits lower bandwidth at -3 dB and higher DC gain. These performance differences result from inaccuracies in the nominal model, particularly neglecting anti-resonance and friction effects observed in the FRF data shown in Fig.~\ref{fig:J2_FRF}. The task-efficient optimized DOB further improves control performance by reducing inertia variation, leading to higher bandwidth and DC gain.

\begin{figure}
	\begin{center}
		\subfloat[]{\includegraphics[width=0.5\textwidth]{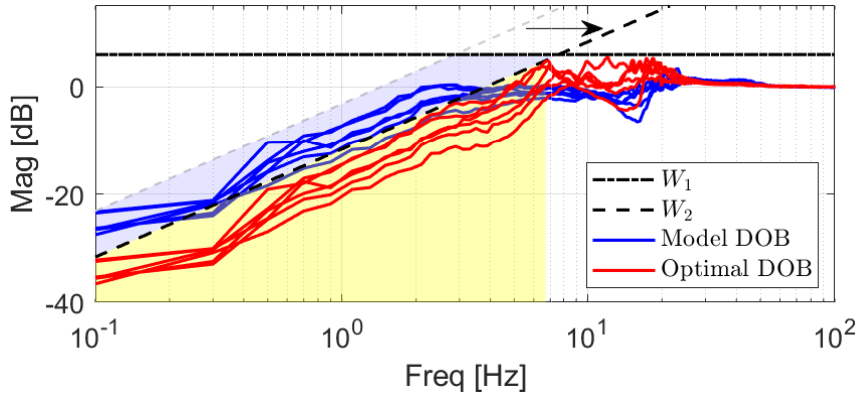}}\\	
        \vspace{-1em}
        \subfloat[]{\includegraphics[width=0.5\textwidth]{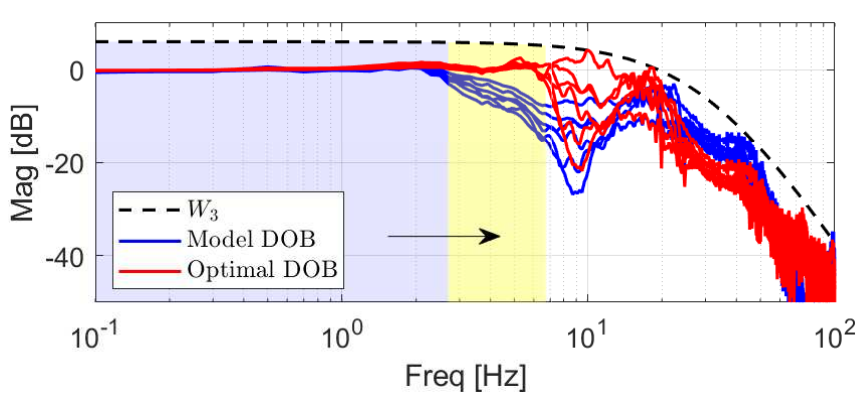}}
		\caption{DOB optimization results for the 2nd joint of the flexible joint robot. (a) Sensitivity function (b) Complementary sensitivity function. The optimization maximizes control performance by achieving higher control bandwidth and reduced overshoot while maintaining the modulus margin.}
        \label{fig:Fig_T_opt}
	\end{center}
\end{figure}

\subsection{Robustness against System Variation} \label{sec:5.4} 

The robustness of the proposed DOB against system variations, particularly changes in link inertia, is experimentally validated. The second joint is chosen for detailed analysis due to its representative dynamic behavior among all joints. Fig.~\ref{fig:Exp_variation} shows the robot configuration transitioning from minimum to maximum inertia, generating disturbances due to inertia variations.

Fig.\ref{fig:Exp_variation_result} compares the proposed DOB against the model-based DOB and no control scenario, displaying the disturbance torque, velocity deviations, and corresponding power spectra. The proposed DOB significantly reduces velocity deviations, as shown in Fig.\ref{fig:Exp_variation_result}(b). The model-based DOB maintains robustness against inertia variations but shows greater deviations due to inaccuracies in its nominal model.

In Fig.\ref{fig:Exp_variation_result}(c), the proposed DOB achieves superior disturbance attenuation from 0.1-7 Hz due to its enhanced rejection capabilities, as demonstrated in Fig.\ref{fig:Fig_T_opt}(a). However, a minor increase in power spectra between 7–8 Hz is observed due to the DOB's overshoot effect.

\begin{figure}
\begin{center}
    \includegraphics[width=1.0\columnwidth]{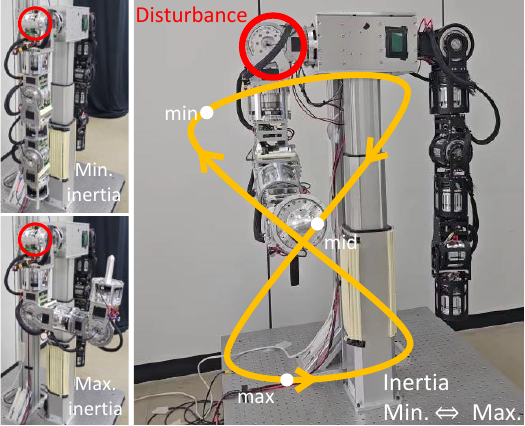}
	\caption{Inertia variation in 2nd joint according to robot movement.}
	\label{fig:Exp_variation}
\end{center}
\end{figure}

\begin{figure}
	\begin{center}
        \subfloat[]{\includegraphics[width=1.0\columnwidth]{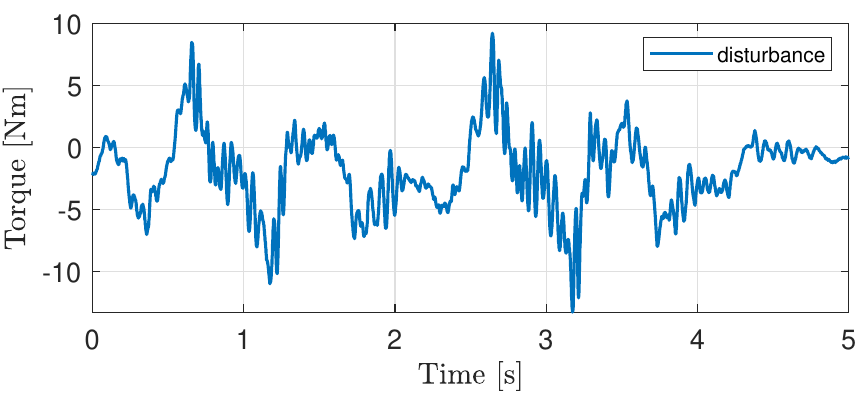}}\\
        \vspace{-1em}
        \subfloat[]{\includegraphics[width=1.0\columnwidth]{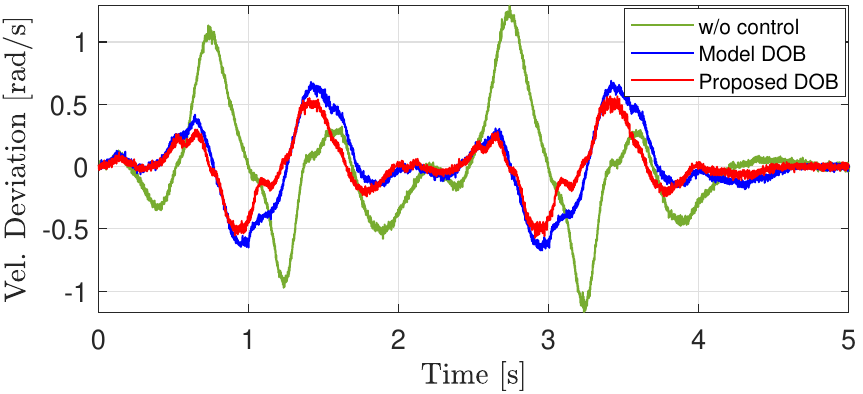}}\\
        \vspace{-1em}
        \subfloat[]{\includegraphics[width=1.0\columnwidth]{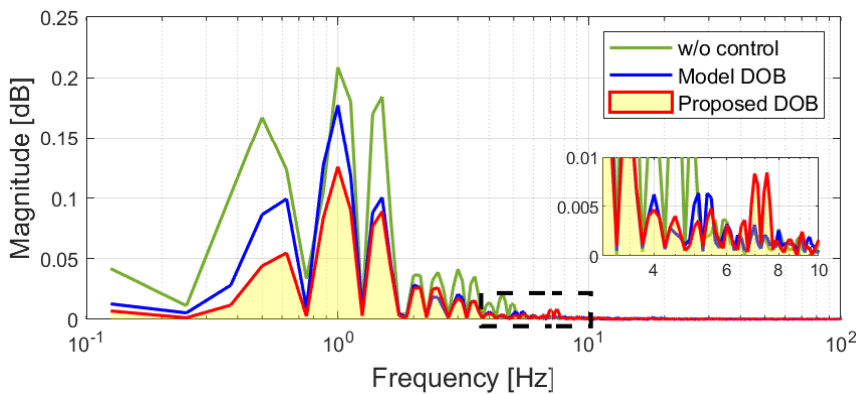}}
		\caption{Robustness against system variation in the 2nd joint: (a) Disturbance induced by inertia variation, (b) velocity deviation, and (c) corresponding power spectra showing disturbance attenuation performance.}
		\label{fig:Exp_variation_result}
	\end{center}
 \vspace{-0.2cm}
\end{figure}

\begin{figure*}[t]
    \centering
    \includegraphics[width=1.0\textwidth]{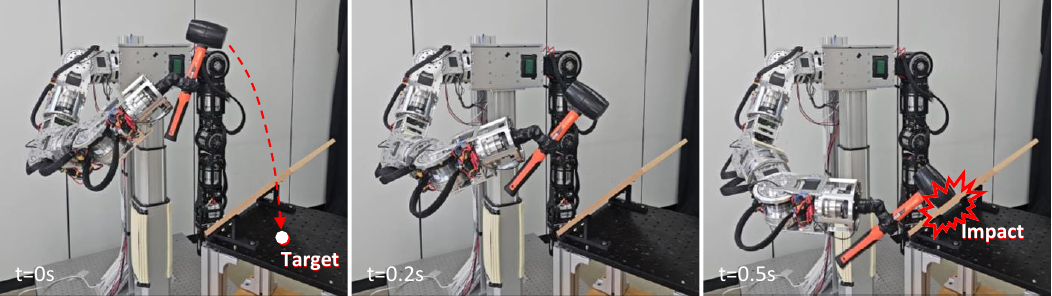}
    \caption{Impact tolerance during a high-speed hammering task. The target point is positioned behind a plate, and the robot strikes the plate at high speed, reaching it within 0.5 seconds.}
    \label{fig:Exp_Hammering}
\end{figure*}

\subsection{Impact Tolerance under High-speed Hammering}  \label{sec:5.5}

\begin{figure}
\begin{center}
        \includegraphics[width=1.0\columnwidth]{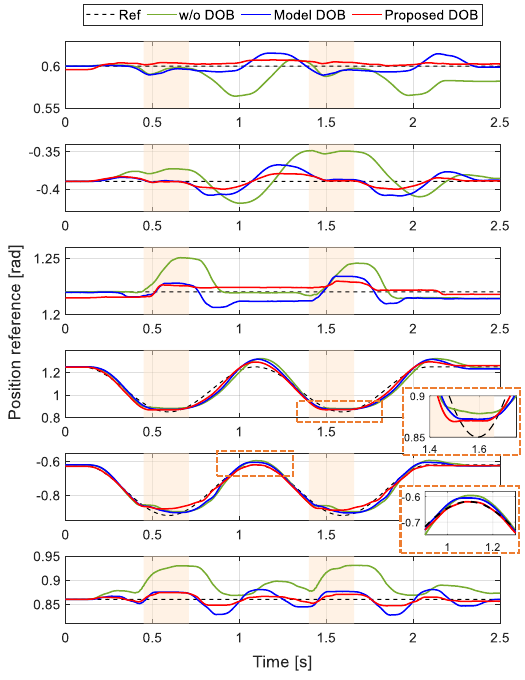}
	\caption{Position Tracking during high-speed hammering. Joints 4 and 5 execute dynamic hammering motions, while other joints remain stationary. The shaded area indicates the moment of impact.}
	\label{fig:Exp_hammering_test}
\end{center}
\end{figure}

Impact tolerance is evaluated through a high-speed hammering task, as depicted in Fig.~\ref{fig:Exp_Hammering}, in which the robot repeatedly impacts a barrier placed in front of a specified target. This task introduces complex disturbances, including impacts, joint coupling effects, and inertia variation.

Joint position tracking performance during the hammering operation is illustrated in Fig.~\ref{fig:Exp_hammering_test}. The implemented control strategy incorporates a DOB combined with feedforward control, utilizing a shared dynamic model with the DOB, complemented by feedback proportional-derivative (PD) control to enhance motion accuracy. Sinusoidal trajectories are prescribed for joints 4 and 5 to simulate the hammering action, with impact occurrences highlighted by shaded regions in the figure.

\begin{figure}
\begin{center}
	\includegraphics[width=1.0\columnwidth]{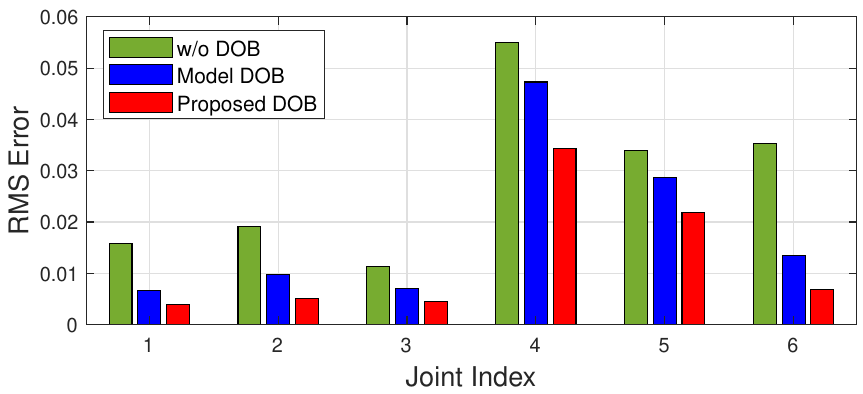}
	\caption{RMS error of position tracking performance}
	\label{fig:RMSE_pos}
\end{center}
\end{figure}

\begin{figure}
\begin{center}
	\includegraphics[width=1.0\columnwidth]{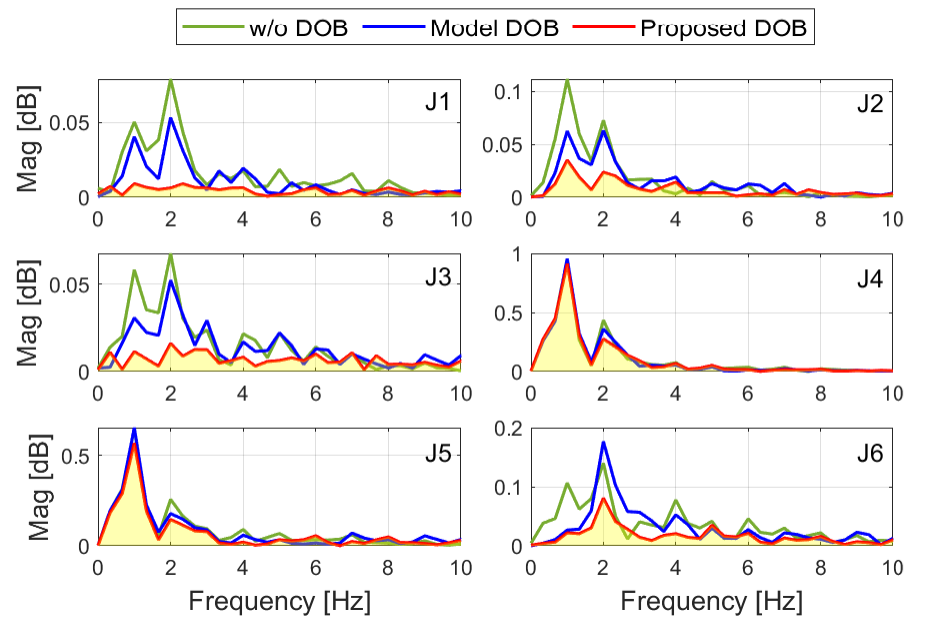}
	\caption{Power spectra of flexible joint robot under high-speed hammering: analysis of joint velocity and torque input.}
	\label{fig:Exp_Hammer_Vel_FRF}
\end{center}
\end{figure}

Notably, joint 4 exhibits significant deviations from its reference trajectory during impact events, particularly noticeable around 1.5 seconds due to barrier-induced obstruction. After each impact, the proposed DOB significantly improves trajectory tracking, notably in joint 5, around 1 second. Compared to the model-based DOB, the proposed DOB demonstrates substantially reduced deviations, particularly in joints 1–3 and 6, which are commanded to maintain stationary positions. This enhanced performance is attributed to the DOB's superior bandwidth and its robust capability in addressing integrated disturbances and uncertainties within the system dynamics.

Quantitative evaluations through Root Mean Square Error (RMSE), shown in Fig.\ref{fig:RMSE_pos}, confirm the superior tracking performance of the proposed DOB. Furthermore, Fig.\ref{fig:Exp_Hammer_Vel_FRF} depicts power spectra from motor torque input to joint velocity, highlighting the effectiveness of the proposed DOB. Specifically, joints 1, 2, 3, and 6 exhibit notably reduced power spectral densities, indicating enhanced control bandwidth. Within the frequency range of 0.1–1 Hz, joints 4 and 5 exhibit comparable spectral characteristics corresponding to the sinusoidal reference frequency. However, the proposed DOB significantly reduces spectral densities above 1 Hz, clearly indicating its superior capability in disturbance rejection.

\section{Conclusion} \label{sec:6}
An auto-tuning methodology for DOB design with optimal performance for FJR is proposed. By extracting system characteristics such as resonance and system variation from FRF, the proposed method optimizes DOB performance, maximizing bandwidth and reducing vibrations. 

The optimization is formulated in the frequency domain using closed-loop transfer functions that account for tracking performance, robustness, and stability. The non-convex optimization problem is reformulated and solved as a convex optimization through convexification. The superiority of optimized DOB is validated experimentally under complex dynamic conditions, including high-speed impacts and inertia variations.

Future research aims to integrate the proposed DOB with advanced feedback and feedforward controls, investigate MIMO identification and optimization for joint coupling, and explore LPV synthesis techniques for improved adaptability and robustness.

\section*{Acknowledgment}

* This work was supported by the National Research Foundation of Korea(NRF) grant funded by the Korea government(MSIT) (No. RS-2024-00354028).

\bibliography{reference}{}

\begin{thebibliography}{10}

\bibitem{de2016robots}
A.~De~Luca and W.~J. Book, ``Robots with flexible elements,'' {\em Springer handbook of robotics}, pp.~243--282, 2016.

\bibitem{spong1987modeling}
M.~W. Spong, ``Modeling and control of elastic joint robots,'' 1987.

\bibitem{makarov2016modeling}
M.~Makarov, M.~Grossard, P.~Rodriguez-Ayerbe, and D.~Dumur, ``Modeling and preview h backslash infty control design for motion control of elastic-joint robots with uncertainties,'' {\em IEEE Transactions on Industrial Electronics}, vol.~63, no.~10, pp.~6429--6438, 2016.

\bibitem{rsetam2019cascaded}
K.~Rsetam, Z.~Cao, and Z.~Man, ``Cascaded-extended-state-observer-based sliding-mode control for underactuated flexible joint robot,'' {\em IEEE Transactions on Industrial Electronics}, vol.~67, no.~12, pp.~10822--10832, 2019.

\bibitem{chen2015disturbance}
W.-H. Chen, J.~Yang, L.~Guo, and S.~Li, ``Disturbance-observer-based control and related methods—an overview,'' {\em IEEE Transactions on industrial electronics}, vol.~63, no.~2, pp.~1083--1095, 2015.

\bibitem{bang2009robust}
J.~S. Bang, H.~Shim, S.~K. Park, and J.~H. Seo, ``Robust tracking and vibration suppression for a two-inertia system by combining backstepping approach with disturbance observer,'' {\em IEEE transactions on industrial electronics}, vol.~57, no.~9, pp.~3197--3206, 2009.

\bibitem{kim2019model}
M.~J. Kim, F.~Beck, C.~Ott, and A.~Albu-Sch{\"a}ffer, ``Model-free friction observers for flexible joint robots with torque measurements,'' {\em IEEE Transactions on Robotics}, vol.~35, no.~6, pp.~1508--1515, 2019.

\bibitem{nozaki2013decoupling}
T.~Nozaki, T.~Mizoguchi, and K.~Ohnishi, ``Decoupling strategy for position and force control based on modal space disturbance observer,'' {\em IEEE Transactions on Industrial Electronics}, vol.~61, no.~2, pp.~1022--1032, 2013.

\bibitem{yun2024disturbance}
T.~H. Yun and M.~J. Kim, ``Disturbance observer with constraints,'' {\em IEEE Control Systems Letters}, 2024.

\bibitem{xie2021disturbance}
H.~Xie, L.~Dai, Y.~Lu, and Y.~Xia, ``Disturbance rejection mpc framework for input-affine nonlinear systems,'' {\em IEEE Transactions on Automatic Control}, vol.~67, no.~12, pp.~6595--6610, 2021.

\bibitem{boyd2016mimo}
S.~Boyd, M.~Hast, and K.~J. {\AA}str{\"o}m, ``Mimo pid tuning via iterated lmi restriction,'' {\em International Journal of Robust and Nonlinear Control}, vol.~26, no.~8, pp.~1718--1731, 2016.

\bibitem{tesfaye2000sensitivity}
A.~Tesfaye, H.~S. Lee, and M.~Tomizuka, ``A sensitivity optimization approach to design of a disturbance observer in digital motion control systems,'' {\em IEEE/ASME Transactions on mechatronics}, vol.~5, no.~1, pp.~32--38, 2000.

\bibitem{karimi2010fixed}
A.~Karimi and G.~Galdos, ``Fixed-order h∞ controller design for nonparametric models by convex optimization,'' {\em Automatica}, vol.~46, no.~8, pp.~1388--1394, 2010.

\bibitem{wang2022frequency}
X.~Wang, W.~Ohnishi, and T.~Koseki, ``Frequency response data based disturbance observer design: With application to a nonminimum phase motion stage,'' {\em IEEE/ASME Transactions on Mechatronics}, vol.~27, no.~6, pp.~5318--5326, 2022.

\bibitem{schuchert2023frequency}
P.~Schuchert and A.~Karimi, ``Frequency-domain data-driven position-dependent controller synthesis for cartesian robots,'' {\em IEEE Transactions on Control Systems Technology}, vol.~31, no.~4, pp.~1855--1866, 2023.

\bibitem{spong2020robot}
M.~W. Spong, S.~Hutchinson, and M.~Vidyasagar, {\em Robot modeling and control}.
\newblock John Wiley \& Sons, 2020.

\bibitem{garcia2004robust}
D.~Garcia, A.~Karimi, and R.~Longchamp, ``Robust pid controller tuning with specification on modulus margin,'' in {\em Proceedings of the 2004 American Control Conference}, vol.~4, pp.~3297--3302, IEEE, 2004.

\bibitem{boyd1994linear}
S.~Boyd, L.~El~Ghaoui, E.~Feron, and V.~Balakrishnan, {\em Linear matrix inequalities in system and control theory}.
\newblock SIAM, 1994.

\bibitem{karimi2017data}
A.~Karimi and C.~Kammer, ``A data-driven approach to robust control of multivariable systems by convex optimization,'' {\em Automatica}, vol.~85, pp.~227--233, 2017.

\bibitem{lee2023exsler}
D.~Lee, K.~Choi, J.~Kim, W.~Yun, T.~Kim, K.~Nam, and S.~Oh, ``Exsler: Development of a robotic arm for human skill learning,'' in {\em 2023 IEEE/ASME International Conference on Advanced Intelligent Mechatronics (AIM)}, pp.~209--214, IEEE, 2023.

\bibitem{sariyildiz2018stability}
E.~Sariyildiz, H.~Sekiguchi, T.~Nozaki, B.~Ugurlu, and K.~Ohnishi, ``A stability analysis for the acceleration-based robust position control of robot manipulators via disturbance observer,'' {\em IEEE/ASME Transactions on Mechatronics}, vol.~23, no.~5, pp.~2369--2378, 2018.

\bibitem{SolidWorks2022}
{Dassault Systèmes}, {\em {SolidWorks 2022}}.
\newblock Dassault Systèmes SolidWorks Corp., Waltham, Massachusetts, USA, 2022.
\newblock CAD software.

\bibitem{pintelon2012system}
R.~Pintelon and J.~Schoukens, {\em System identification: a frequency domain approach}.
\newblock John Wiley \& Sons, 2012.

\end{thebibliography}
\bibliographystyle{ieeetr}

\addtolength{\textheight}{-12cm}   % This command serves to balance the column lengths
\end{document}